\documentclass[10pt,twocolumn,letterpaper]{article}

\usepackage[algorithms]{wacv}      %
\usepackage[pagebackref,breaklinks,colorlinks,allcolors=wacvblue]{hyperref}

\definecolor{wacvblue}{rgb}{0.21,0.49,0.74}
\usepackage[pagebackref,breaklinks,colorlinks,allcolors=wacvblue]{hyperref}
\usepackage{amsthm}
\usepackage{multirow}  %
\usepackage{array}     %
\usepackage{comment}

\usepackage{hyperref}       %
\usepackage{url}            %
\usepackage{booktabs}       %
\usepackage{amsfonts}       %
\usepackage{nicefrac}       %
\usepackage{microtype}      %
\usepackage{xcolor}         %
\usepackage{comment}
\usepackage{amsthm}
\usepackage{multirow}  %
\usepackage{array}     %
\usepackage{comment}
\usepackage{amsmath}
\usepackage{mathtools}

\newtheorem{proposition}{Proposition}

\newtheorem{lemma}{Lemma}

\def\wacvPaperID{*****} %
\def\confName{WACV}
\def\confYear{2026}

\title{CLARGA: Multimodal Graph Representation Learning over Arbitrary Sets of Modalities}
\author{Santosh Patapati\\
\\
}

\begin{document}
\maketitle
\begin{abstract}
We introduce CLARGA, a general-purpose multimodal fusion architecture for multimodal representation learning that works with any number and type of modalities without changing the underlying framework. Given a supervised dataset, CLARGA can be applied to virtually any machine learning task to %
fuse different multimodal representations for processing by downstream layers. On a sample-by-sample basis, CLARGA learns how modalities should inform one another by building an attention weighted graph over their features and passing messages along this graph with a multi-head Graph Attention Network. Not only does this make CLARGA highly adaptive, as it constructs unique graphs for different samples, it makes for efficient fusion with sub-quadratic complexity as the number of modalities grows. Through a learnable mask, it can also adapt to missing modality inputs. The model is trained with a hybrid objective that combines a supervised task loss with contrastive InfoNCE loss, improving cross-modal consistency and robustness to noisy inputs. We demonstrate CLARGA's effectiveness in diverse multimodal representation learning tasks across 7 datasets spanning finance, human-computer interaction, general multimedia classification, and affective computing. It consistently outperforms baselines, state-of-the-art models, and ablations. Additional experiments also demonstrate its robustness to missing inputs and ability to excel on niche tasks.
Overall, CLARGA can be easily plugged into machine learning models for effective and efficient learning of representations %
across a 
wide
variety of tasks. %

\end{abstract}

\section{Introduction}
Multimodal data—images, audio, text, time-series, and more—are now commonplace across a variety of domains \cite{baltrusaitis2018multimodal}. Effectively representing and fusing information that arrives in different sensory or semantic formats has therefore become a core challenge of machine learning, with a wide range of applications including chemistry \cite{li2024chemvlm}, physics \cite{huerta2019enabling}, and healthcare \cite{kline2022multimodal}. %
Naïve strategies that simply concatenate feature vectors ("early fusion") or average modality-level decisions ("late fusion") often struggle with heterogeneity, quadratic growth in parameters, and sensitive behavior when some inputs are missing \cite{atrey2010multimodal, snoek2005early, baltrusaitis2018multimodal}. Likewise, end-to-end attention-based fusion models typically require paired data for every modality at training time \cite{tsai2019multimodal, hori2017attentionbased}. Such models may scale quadratically with the number of modalities, making them costly or infeasible for real-world applications \cite{nagrani2021attention, golovanevsky2023one}.

We introduce Contrastive Learning with Adaptive Residual Graph Attention (CLARGA), a general-purpose fusion architecture that accepts arbitrary numbers and types of modalities. It models modality's embeddings as different nodes in a fully connected graph. 
We use a lightweight, multi-head Graph Attention Network (GAT) \cite{velivckovic2018graph} to learn how strongly nodes should connect across modalities. This lets each modality selectively draw information from the others while keeping time complexity below quadratic in the number of nodes, so the model scales to larger graphs. We add residual connections and Layer Normalization to make deeper message passing stable and to reduce oversmoothing, where node representations become indistinguishable. During training, CLARGA optimizes a hybrid objective. That is, a standard supervised loss for the task combined with an InfoNCE contrastive term that essentially encourages each unimodal embedding to agree with the fused representation. This contrastive component helps align modalities and makes the fused features more robust when some inputs are noisy or incomplete.

Our contributions are as follows:

\begin{enumerate}
    \item We introduce CLARGA, a graph-based %
    framework for multimodal representation learning that works with {any} number and type of modalities without changing the architecture.
    \item We couple this adaptive graph fusion with a lightweight contrastive alignment loss (InfoNCE) that pulls each unimodal embedding toward the fused representation, promoting agreement across modalities.
    \item We handle missing inputs with a learnable {mask token} that stands in for absent modalities and is processed by the same attention mechanism, making fusion robust to missing or noisy data. %
    \item We construct novel proofs that build on existing theory. Specifically, we provide: (i) universality of the fusion block as an approximator of any continuous permutation-invariant function; (ii) a Lipschitz bound quantifying the impact of a missing modality on the fused vector; and (iii) a Rademacher complexity bound for the joint supervised contrastive loss.
    \item Across seven public benchmarks, CLARGA consistently outperforms strong baselines, state-of-the-art fusion methods, and ablations, also achieving state-of-the-art results on the DAIC-WoZ dataset.
\end{enumerate}

\section{Related Works}
\subsection{Early Fusion Architectures}
Early work used feature-level fusion%
, where raw features from all modalities are concatenated and fed into a single model, and decision-level fusion%
, where each modality is processed by a separate model and their final predictions are aggregated \cite{Poria2016, Roitberg2022}. 
Despite their simplicity, both approaches struggle with heterogeneity, quadratic growth of parameters, and missing inputs.

\subsection{Attention-based \& Transformer-style Fusion}
Recent systems are beginning to use learned attention mechanisms that let the model weight modalities on a per-sample basis. Crossmodal attention \cite{Wei2020} %
has become popular for fusing information across modalities.
MulT \cite{tsai2019multimodal} injects attention across every stream so that one modality can guide another without explicit temporal alignment.
(Nagrani et al., 2021) \cite{nagrani2021attention} builds on this with the Multimodal bottleneck Transformer to funnel all crossmodal interactions through a few shared bottleneck tokens per layer, compressing attention and cutting its quadratic cost while creating rich crossmodal exchange.

\subsection{Graph-based Fusion}
Graph-based architectures represent heterogeneous inputs as nodes and edges in a graph, enabling flexible relational modeling. Graph Attention Networks \cite{velivckovic2018graph} extend this by learning attention weights on edge. %
Graph-based fusion, a newer approach, begins by building a graph in which nodes encode information about each modality (e.g., entire audio or video modalities), or even finer-grained elements (e.g., objects or words) \cite{ektefaie2023multimodallearninggraphs, huang2024multimodalrepresentationlearningusing, Li2022GraphMFT}. %
Nodes are connected by edges which encode the relationships between the modalities. %

\subsection{Learning Shared Representations}
Beyond task-specific fusion, %
models can learn modality-agnostic embeddings that can be reused across tasks. Deep CCA \cite{Andrew2013, Benton2017} trains per-modality encoders to maximize correlation of paired samples. Other variants combine this with task-specific losses (e.g., ViLBERT \cite{Lu2019}) to preserve modality-specific details while aligning shared structures.

\subsection{Contrastive Alignment}
Contrastive learning %
has delivered  gains in zero-shot transfer across many tasks \cite{jang2024significantly, patapati2025clipmgguidingsemanticattention, Sinha2024, Lu2024}. CLIP \cite{Radford2021CLIP} aligns image and text encoders by pulling matches pairs together in embedding space and pushing mismatches apart. 
For example, DGI \cite{Velickovic2019DGI} uses a contrastive InfoMax objective \cite{Hjelm2018DeepInfoMax} to maximize mutual information between local node patch embeddings and global graph summary for unsupervised representation learning.

\subsection{Handling Missing Modalities}
Real-world deployments can suffer from missing modalities, which model architectures have to account for. 
Many models that address this issue employ imputation \cite{Tran2017, Zhang2023, Poudel2025}, robust training \cite{Ma2022, Woo2022}, or shared-specific factorization \cite{Tsai2018, Wang2023}. Across these methods, learned mask tokens and contrastive alignment losses have emerged as lightweight yet effective tools. The former gives models an explicit symbol for missing data \cite{Ramazanova2025, Maheshwari2024, Kim2024MMP}, while the latter keeps partial-input embeddings close to the manifold learned from complete data \cite{Lin2023MissModal, He2023COM, Ghorbani2023}. Both have been demonstrated to be effective independently.

\section{Methodology}
\subsection{Problem Setup}

\subsubsection{Data and Notation}
We consider a supervised dataset

$$
\mathcal D=\Bigl\{(x^{(n)}_1,\dots,x^{(n)}_M,\;y^{(n)})\Bigr\}_{n=1}^{N},
$$

where each example may contain up to $M$ heterogeneous modalities. For the $m$-th modality, we denote the input domain by $\mathcal X_m$. A modality encoder

$$
f_m:\;\mathcal X_m \;\longrightarrow\; \mathbb R^{d}
$$

maps raw input to a $d$-dimensional latent representation

$$
h_m = f_m\!\bigl(x_m\bigr).
$$

Missing modalities are common in real data.  Whenever modality $m$ is absent we substitute $h_m$ with a learned mask embedding $h_{\text{mask}}\in\mathbb R^{d}$. A diagonal binary mask $\mathbf{r}\in\{0,1\}^M$ (where $r_i{=}1$ denotes missing)\footnote{We use $M=\mathrm{diag}(\mathbf{r})$ for convenience in our later discussions} is stored so that the graph attention block can conveniently ignore self edges of masked nodes (details in $\S3$)..

\subsubsection{Goal}
Our aim is to learn the parameter set

$$
\theta=\Bigl\{f_{1},\dots,f_{M},\;W_q, W_k, W_g,\;q_F,\;g\Bigr\},
$$

where: (1) $W_q,W_k$ project node features to query / key space for the multi-head graph attention layers; (2) $W_g$ contains the weights of the residual GAT message passing components; (3) $q_F$ is a learnable query vector that aggregates node embeddings into the global fusion vector $z_{\text{fusion}}$; and (4) $g$ is a shallow prediction head.

The parameters are optimized for predictive accuracy and alignment of modalities during fusion, so that: (1) The fused prediction $\hat y=g(z_{\text{fusion}})$ minimizes the task loss on the training set; (2) Each modality embedding $h_m$ shares high mutual information with $z_{\text{fusion}}$, enforced through a batchwise InfoNCE objective.

\subsection{Proposed Approach}
In this section, we detail the full CLARGA framework.%

\subsubsection{Modality Encoders}

For each of the $M$ input modalities (e.g. image, audio, text, tabular, timeseries), we employ a dedicated encoder that projects raw data into a shared $d$-dimensional feature space\footnote{The encoders employed in experimentation include ResNet, 1D-CNNs, BERT-based encoder models \cite{devlin-etal-2019-bert}, and more. The encoders can be either trainable or frozen (see Section 4).}. 

If a modality is missing for a given sample, we substitute its embedding by a learnable mask vector $h_{\mathrm{mask}}\in\mathbb R^{256}$.  A binary mask tensor accompanies the embeddings so later layers can ignore unusable nodes.

\subsubsection{Adaptive Graph Attention}
The $M$ modality vectors $\{h_i\}_{i=1}^M$ form the nodes of a fully connected and directed graph that is specialized for the specific sample at hand.  
To quantify how much information each modality should gather from every other, we employ a multi-head attention mechanism.

First, each node is projected into a shared query key space:

$$
q_i = W_q h_i,
\qquad
k_j = W_k h_j,
\qquad
W_q, W_k \in \mathbb{R}^{128\times 256}.
$$

For every attention head $h\in\{1,\dots,H\}$ we compute the raw compatibility score

$$
e_{ij}^{(h)} \;=\; \mathrm{LeakyReLU}\!\bigl(q^{(h)\top}_i k^{(h)}_j\bigr).
$$

To prevent a node from redundantly attending to itself and to exclude missing modalities, diagonal scores and rows corresponding to masked inputs are set to $-\infty$.  The softmax operation along each destination node then yields normalized attention coefficients.

$$
\alpha_{ij}^{(h)} \;=\;
\frac{\exp\bigl(e_{ij}^{(h)}\bigr)}
     {\displaystyle \sum_{k\neq i}\exp\bigl(e_{ik}^{(h)}\bigr)}.
$$

The resulting $\alpha_{ij}^{(h)}$ values determine, on a sample by sample basis, the strength with which modality $j$ influences modality $i$ during subsequent message passing layers.

\subsubsection{Residual Graph Attention Layers}
To propagate information across modalities we apply a stack of $D$ residual graph attention layers\footnote{In our experiments, we set $D=3$.}. At layer $\ell\in\{0,\dots,L-1\}$ each node $i$ aggregates messages from its neighbours via the attention coefficients that are averaged across heads (introduced in $\S3.2.2$).

$$
m_i^{(\ell)} \;=\; \frac{1}{H}\sum_{h=1}^{H}\sum_{j=1}^{M}\alpha_{ij}^{(h)}\,h_j^{(\ell)} .
$$

The aggregated message is concatenated with the node's current state and linearly transformed,

$$
\tilde h_i^{(\ell)} \;=\; W_g\!\bigl[h_i^{(\ell)} \Vert m_i^{(\ell)}\bigr],
\qquad W_g \in \mathbb{R}^{256\times 512},
$$

after which a residual connection and Layer Normalization produce the updated embedding,

$$
h_i^{(\ell+1)} \;=\;
\mathrm{LayerNorm}\!\bigl(h_i^{(\ell)} + \mathrm{Dropout}(\tilde h_i^{(\ell)})\bigr).
$$

\subsubsection{Fusion Read-Out}
Once message passing is complete, the model must collapse the $M$ context enriched node embeddings into a single multimodal representation.  We introduce a learnable query vector $q_F\in\mathbb{R}^{128}$ and compute scalar scores of relevance.

$$
s_i \;=\; q_F^{\!\top} W_k h_i^{(D)},
$$

which are converted to attention weights $\beta_i = \mathrm{softmax}(s_1,\dots,s_M)$.  The final fused vector is the corresponding weighted sum:

$$
z_{\text{fusion}} \;=\; \sum_{i=1}^{M} \beta_i\,h_i^{(D)} .
$$

A dropout layer with probability $0.1$ is applied to $z_{\text{fusion}}$ before it enters the task specific prediction head.

\subsubsection{Optimization Objective}
CLARGA is trained with a dual term loss that couples supervised learning with an alignment loss for modality fusion.

Firstly, we incorporate task loss. For a labeled example $(x,y)$ the supervised term is

$$
\mathcal{L}_{\mathrm{sup}} \;=\;
\begin{cases}
\mathrm{CE}\!\bigl(g(z_{\text{fusion}}),\,y\bigr), & \text{classification},\\[6pt]
\mathrm{MSE}\!\bigl(g(z_{\text{fusion}}),\,y\bigr), & \text{regression},
\end{cases}
$$

where $g$ denotes the shallow prediction head.

Secondly, to ensure that every modality remains well aligned with the fused representation, we adopt the InfoNCE contrastive loss \cite{oord2019representationlearningcontrastivepredictive} with batch negatives. %

\subsection{Theoretical Analysis}
We establish three properties for CLARGA: 1) Expressivity, the fusion block is a universal approximator for any continuous, order agnostic function over a set of heterogenous modalities. That is, it can represent any continuous permutation invariant mapping on modality sets (§3.3.1); 2) Robustness to missing inputs, as the architecture is Lipchitz robust to dropping a single modality. If one modality is removed, the change in the fused representation (and therefore in the prediction) is bounded and scales proportionally with the norm of the missing input (§3.3.2); and 3) The generalization of the training objective. The hybrid supervised and contrastive loss has a bound on excess risk of order $O!\big(\sqrt{\log \mathcal{N}(\varepsilon)/n}\big)$, where $\mathcal{N}(\varepsilon)$ is a covering number. This essentially gives $1/\sqrt{n}$ generalization scaling controlled by a term for data complexity (§3.3.3).
We also provide commentary on existing proofs and literature regarding InfoNCE alignment loss ($\S3.3.4$) and on how residual and LayerNorm GAT layers prevent oversmoothing ($\S3.3.5$).

\subsubsection{Universality of the CLARGA Fusion Block}
\begin{proposition}
\label{prop:univ_fusion}

Let 

$$
f:(\mathbb R^{d})^{M}\rightarrow\mathbb R^{p}
$$ 

be any continuous permutation invariant function.
For every compact $\mathcal K$ and every $\varepsilon!>!0$, there exists a choice of weights $\theta$ in a multi-head CLARGA fusion block such that

$$
\sup_{x\in\mathcal K}
  \bigl\|{\mathrm{CLARGA}}_{\theta}(x)-f(x)\bigr\|<\varepsilon .
$$

\end{proposition}

\begin{proof}[Proof sketch]

\begin{enumerate}
    \item \textbf{Deep Sets form.} (Zaheer et al., 2017) shows any continuous, permutation-invariant $f$ can be decomposed as $\rho\bigl(\sum_i\phi(x_i)\bigr)$ \cite{Zaheer2017DeepSets}.
    \item \textbf{Attention subsumes summation.} A multi-head GAT with shared query and key projections computes $\sum_i\alpha_i\phi(x_i)$. Setting all logits equal forces $\alpha_i=1/M$, recovering the Deep Sets sum. Learnable logits therefore strictly enlarge the function class.
    \item \textbf{Continuity and invariance.} Because the softmax is continuous and symmetric, the map $(x_1,\dots,x_M)\mapsto\sum_i\alpha_i\phi(x_i)$ stays inside the invariant function space $\mathcal S$.
    \item \textbf{Density preservation.} Composing with a universal MLP $\rho$ maintains density in $\mathcal S$ \cite{Cybenko1989,Hornik1989}.
\end{enumerate}

Therefore the CLARGA fusion block is dense in $\mathcal S$, extending universality theory regarding invariant networks for attention-based fusion \cite{Maron2019}. The full proof and explanation with additional architectural information is written in Appendix B.
\end{proof}

\subsubsection{Lipschitz Robustness to Missing Modalities}
\begin{proposition}[Lipschitz robustness]\label{prop:lip_robust}
Assume each encoder $f_m$ is $L$ Lipschitz and that the fusion weights satisfy $\sum_i\beta_i=1$ with $\beta_i\ge0$.
Replacing a single modality $k$ by the mask embedding $h_{\mathrm{mask}}$ yields

$$
\bigl\|z_{\mathrm{fusion}}^{\mathrm{full}}-z_{\mathrm{fusion}}^{\mathrm{masked}}\bigr\|
\;\le\;
L\,\beta_k\,\|x_k\|.
$$

If the task head $g$ is further constrained to be $K$ Lipschitz (e.g. via spectral normalization), the prediction perturbation follows

$$
\bigl\|g(z^{\mathrm{full}})-g(z^{\mathrm{masked}})\bigr\|
\le
K L\,\beta_k\,\|x_k\|.
$$

\end{proposition}

\begin{proof}[Proof sketch]
The encoder perturbation obeys $|h_k-h_{\mathrm{mask}}|\le L|x_k|$. All other encoders remain unchanged. Graph attention layers are 1-Lipschitz (i.e., the output change is at most the input change) 
when attention coefficients are treated as fixed in the forward pass \cite{Arghal2021}. Because each subsequent residual GAT layer and LayerNorm does not expand, the perturbation norm after $D$ layers is still bounded by $L|x_k|$. Finally, the fusion step is a convex combination with coefficient $\beta_k$, scaling the deviation by at most $\beta_k$. The classifier contributes a multiplicative $K$ factor, completing the bound. Full proof appears in Appendix C. %
\end{proof}

\subsubsection{Generalization Bound for Supervised{--}Contrastive Objective}
\begin{proposition}[Rademacher complexity bound]\label{prop:gen_hybrid}
Let $\mathcal H$ be the class of CLARGA networks whose parameter matrices have Frobenius norm bounded by $B$ and whose activation functions are $1$-Lipschitz. Let $\widehat h$ minimize the hybrid loss $\mathcal L = \mathcal L_{\mathrm{sup}}+\lambda_c\mathcal L_{\mathrm{NCE}}$ over $n$ 
independent and identically distributed examples. Then, for any $0<\delta<1$, with probability at least $1-\delta$,

$$
\mathcal E(\widehat h)-\mathcal E^{\star}
\;\le\;
\tilde{O}\!\Bigl(
  \sqrt{\frac{B^{2}\,d_{\mathrm{eff}}\,+\,\log(1/\delta)}{n}}
\Bigr),
$$

where $d_{\mathrm{eff}}$ is the effective rank of the network’s Jacobian and
$\tilde{O}$ hides poly log factors in the batch size used for negatives.
\end{proposition}

\begin{proof}[Proof sketch]

\begin{enumerate}
    \item \textbf{Hybrid loss Lipschitzness.} Both $\mathcal L_{\mathrm{sup}}$ (cross-entropy with bounded logits) and $\mathcal L_{\mathrm{NCE}}$ (softmax with temperature) are $1$-Lipschitz in the network outputs given weights normalized by their spectral norm.
    \item \textbf{Rademacher complexity.} The Rademacher complexity of $\mathcal H$ is bounded by $\tilde O!\bigl(B\sqrt{d_{\mathrm{eff}}/n}\bigr)$ following (Bartlett and Mendelson, 2002) \cite{Bartlett2002}.
    \item \textbf{McDiarmid concentration.} Lipschitz continuity of the hybrid loss ensures a standard concentration inequality, giving the stated bound of high probability.
\end{enumerate}

A full proposition, group of lemmas, and formal proof are written in Appendix D.

\end{proof}

\subsubsection{Mutual Information View}
InfoNCE is a variational lower bound on the mutual information (MI) between two variables. With an optimal critic (the scoring function in the contrastive loss), this bound is tight \cite{oord2019representationlearningcontrastivepredictive}.
In CLARGA the critic is the cosine similarity in the fused space, so maximizing $\mathcal L_{\text{NCE}}$ encourages each modality embedding $h_m$ to retain information predictive of $z_{\text{fusion}}$. %
A complete proposition, formal proof, and additional commentary are written in Appendix E.

\subsubsection{Depth, Residual Connections, and Oversmoothing}
Finally, residual connections and LayerNorm provably mitigate oversmoothing in linearized GNNs \cite{scholkemper2025residualconnectionsnormalizationprovably}. Since each CLARGA layer matches the standard residual LayerNorm, the lower bound of (Scholkemper et al., 2025) ensures that node representations cannot collapse entirely. %
A complete proposition, formal proof, and additional commentary are written in Appendix F.

\section{Experimental Setup}
\subsection{Datasets} %

\begin{table}[ht]
  \centering
  \caption{Information on the datasets utilized in experimentation}
  \label{tab:datasets}
  \renewcommand{\arraystretch}{1.22}
  \setlength{\tabcolsep}{3pt}  
  \resizebox{\columnwidth}{!}{%
  \begin{tabular}{llll}
    \toprule
    Dataset   & Prediction Task               & Modalities      & Count \\
    \midrule
    AV-MNIST     & Digit     & {Image, Audio}     & 56.0k      \\
    MM-IMDb     & Movie Genre     & {Image, Text}     & 25.9k      \\
    STOCKS-F\&B     & Stock Returns     & Timeseries      & 75.5k      \\
    STOCKS-HEALTH     & Stock Returns     & Timeseries     & 75.5k      \\
    STOCKS-TECH     & Stock Returns     & Timeseries     & 75.5k      \\
    ENRICO     & User Interface     & Image, Set     & 1{,}460      \\
    DAIC-WoZ     & Depression     & Video, Audio, Text     & 189      \\
    \bottomrule
  \end{tabular}
  }
\end{table}

We evaluate CLARGA across a diverse range of datasets to evaluate its generalization and robustness across a variety of applications (see Table \ref{tab:datasets}).

\subsubsection{AV-MNIST}
The Audio Visual-MNIST (AV-MNIST) \cite{vielzeuf2019centralnet} dataset contains spoken audio and image pairs for digit classification tasks. It is a synthetic benchmark where each sample pairs highly noisy MNIST images \cite{lecun1998gradient} and TIDIGITS audio \cite{leonard1993tidigits-ldc}, making it far more difficult than the original MNIST dataset.

For all architectures evaluated on AV-MNIST, we employ a trainable encoder that has not received any pretraining to conduct the initial processing of each modality. The images are encoded using a 4-layer CNN and the spectrograms a 2-layer CNN. They are both finally projected by fully connected layers for processing by the evaluated models. Here, we opt to use smaller, trainable models rather than pretrained models to better isolate the performance of CLARGA and other architectures themselves.

\subsubsection{MM-IMDb}
From the Multimodal-IMDB (MM-IMDb) \cite{arevalo2017gated} dataset, we extract poster images and plot summaries for every movie provided in the dataset. Images and summaries are encoded by a VGG16 \cite{simonyan2015very} and Google word2vec \cite{GoogleWord2Vec} model before being passed into the evaluated architecture for movie genre classification from 23 options.

\subsubsection{STOCKS}
The STOCKS datasets, introduced in (Liang et al., 2021) \cite{Liang2021}, are collections of stock market timeseries data across three categories. Namely: (1) STOCKS-F\&B, which has 14 input and 4 output stocks in the GICS Restaurants or Packaged Food \& Meats category \cite{MSCI2024GICS}, (2) STOCKS-HEALTH, which consists of 56 input and 7 output stocks in the Health Care category, and (3) STOCKS-TECH, which has 94 input and 6 output stocks categorized by GICS as Information Technology or Communication Services.

Every input stock (consisting of 500 trading days) is treated as a separate timeseries mode, with the goal of predicting returns over the next day. To adapt the task to the baseline and state-of-the-art models, we discretize the continuous return variable $R$ into three non-overlapping categories: (1) $Low$, where $0 \le R < 0.1$, (2) $Medium$, where $0.1 \le R < 0.5$, and (3) $High$, where $R \ge 0.5$. Mean Absolute Error (MAE) is calculated by mapping the three classes to numbers ($Low \rightarrow 0$, $Medium \rightarrow 1$, $High \rightarrow 2$) and then deriving MAE as usual. Each modality is encoded by the same CNN-BiLSTM network, which consists of 3 CNNs, 1 BiLSTM \cite{cui2019deepbidirectionalunidirectionallstm}, and one fully connected layer acting as projection.

\subsubsection{ENRICO}
ENRICO \cite{Leva2021} is a higher quality subset of the RICO dataset \cite{Deka2017} consisting of Android app screens categorized by their design topics. We extract UI screenshots and view hierarchy from the dataset. The view hierarchy is treated as a set as it contains an unordered collection of UI elements that each contain metadata and their spatial and structured layout \cite{Liang2021}. 

A frozen pretrained ResNet-18 \cite{he2016residual} model with its head replaced by a projection layer is used for encoding. We employ a frozen pretrained model as an encoder here due to the relatively small size of ENRICO and the level of complexity the task already contains.

\subsubsection{DAIC-WoZ}
The DAIC-WoZ dataset \cite{gratch2014} consists of data from 189 psychotherapy clinical interview recordings. Every recording is accompanied by a Patient Health Question-8 (PHQ-8) \cite{kroenke2009phq8} score, a common inventory used in psychiatry \cite{Shin2019ComparisonPHQ8PHQ9, AriasDeLaTorre2023ReliabilityPHQ8}, which is used to classify the associated participant as either depressed or not depressed.

DAIC-WoZ faces a major data scarcity issue (see Table \ref{tab:datasets}). To mitigate this, we train the proposed approach on 8 second segments and augment the training split using the techniques proposed %
in a recent paper \cite{patapati2024integratinglargelanguagemodels}. To further mitigate this issue, we run 10 fold cross-validation and average the results across all runs \cite{stone1974cross}.

We apply four encoders across three modalities when training and testing CLARGA on the DAIC-WoZ. Facial Action Units (FAUs) \cite{Ekman1978FACS, amos2016openface} are precomputed and used as input for a BiLSTM model to encode the video modality. Text transcripts are processed by MentalRoBERTa \cite{ji2022} and followed by a fully connected layer. Audio is converted into Mel Frequency Cepstral Coefficients (MFCCs) \cite{Stevens1937MelScale, Abdul2022} that are processed by a BiLSTM model. Audio is also processed by wav2vec and a subsequent fully connected layer. Everything except for the pretrained backbone is unfrozen and trainable. These are all then processed by CLARGA for final classification.

Due to the level of compute necessary for running machine learning on the DAIC-WoZ, we train and evaluate only CLARGA and compare it against models specifically designed for the dataset to analyze its ability to perform in highly specific downstream tasks.

\subsection{Baseline and State-of-the-Art Models}
We train and evaluate 4 models on the benchmark datasets to compare against CLARGA. Two of these models are custom baselines used to establish simple, modality agnostic fusion strategies for benchmarking, enabling us to quantify the gains from CLARGA’s graph attention fusion over simple concatenation and averaging approaches. Namely, an early fusion baseline that encodes each modality with its own encoder and concatenates the resulting feature vectors into a joint representation; and an averaged late fusion baseline that trains a separate classifier for each modality and produces the final prediction by averaging each modalities' class probabilities. The other two, Multimodal Lego (MM-Lego) \cite{hemker2025multimodallegomodelmerging} and FuseMix \cite{vouitsis2024dataefficientmultimodalfusionsingle}, are recent state-of-the-art architectures for multimodal processing and representation learning. We provide a more detailed discussion of the models used for comparison in Appendix A.

\subsection{Ablation Study} %

To isolate the effects of different components within CLARGA on overall performance, we conduct an ablation study across four ablations.

\subsubsection{Uniform Attention}
To isolate the benefit of learning attention specifically for each example, we replace the adaptive weights $\alpha_{(ij)}$ with uniform weights, setting

$$
\alpha_{ij} = \frac{1}{M - 1}\quad\text{for all }j \neq i.
$$

This static graph forces each modality to contribute equally.
By comparing the uniform $\alpha$ variant to the full CLARGA with learned attention, we can see the extent to which adaptive weighting improves fusion. 

\subsubsection{No Residual Connections}
To isolate the effect of the skip connection, we remove the residual term in each GNN layer so that

$$
h_i^{(\ell+1)} = \sigma\bigl(W_g \, m_i^{(\ell)}\bigr)
$$

with no added $h_i^{(\ell)}$. This forces each layer to rely solely on aggregated neighbor messages and tests how much the residual helps prevent oversmoothing. 

\subsubsection{No Contrastive Alignment}
We set the InfoNCE weight to zero ($\lambda_{(c)} = 0$), so the model is trained purely with the supervised loss. This removes the optimization for crossmodal alignment and tests how much the contrastive term regularizes learning. 
By comparing this variant to the full CLARGA, we measure how much the InfoNCE objective improves generalization.

\subsubsection{Early Fusion (Mean)}
We test a simple early fusion variant of CLARGA by averaging all initial nodes before any graph processing:

$$
z_{\text{fusion}}^{(0)} = \frac{1}{M}\sum_{n} h_{n}^{(0)}.
$$

We then treat $z_{\text{fusion}}^{(0)}$ as a graph with a single node and no edges, and we pass it directly through the decoder. This isolates the benefit of the graph-based message passing and attention fusion.

\subsection{Modality Dropping Robustness Test}
To assess robustness to missing information, we evaluate CLARGA, every ablation, and MM-Lego under three scenarios on AV-MNIST: (1) all modalities (audio and image) present, (2) image modality dropped at test time, (3) audio modality dropped at test time. For each scenario, the corresponding input is masked out. We report classification accuracy for each setting.
We are unable to perform this experiment for FuseMix as it %
is unable to handle missing modalities.

\section{Results and Discussion}
\begin{table*}[t]
  \centering
  \caption{Performance of Models Across Datasets and GFLOPs Calculated on AV-MNIST (focusing only on method-specific components)}
  \label{tab:results-wide}
  \renewcommand{\arraystretch}{1.22}
  \resizebox{\dimexpr\textwidth\relax}{!}{%
    \begin{tabular}{l 
                c c c      
                *{3}{cc}   
                c          
               }
      \toprule
      \multirow{2}{*}{\textbf{Model}} 
        & \multicolumn{3}{c}{\textbf{Accuracy (\%)}} 
        & \multicolumn{2}{c}{\textbf{F\&B}} 
        & \multicolumn{2}{c}{\textbf{HEALTH}} 
        & \multicolumn{2}{c}{\textbf{TECH}}
        & \multirow{2}{*}{\textbf{GFLOPs}} \\
      \cmidrule(lr){2-4} \cmidrule(lr){5-6} \cmidrule(lr){7-8} \cmidrule(lr){9-10}
        & MM-IMDb & ENRICO & AV-MNIST 
        & Acc. (\%) & MAE 
        & Acc. (\%) & MAE 
        & Acc. (\%) & MAE
        & \\
      \midrule
      Early Fusion
        & 61 & 63   & 66     
        & 51 & 0.53
        & 58 & 0.50 
        & 64 & 0.45
        & 1.214 \\
    \addlinespace[0.1em]
      Late Fusion
        & 56 & 65   & 69     
        & 53 & 0.51 
        & 59 & 0.49 
        & 63 & 0.46
        & 1.214 \\
    \addlinespace[0.1em]
      Multimodal Lego
        & 66 & 76   & \textbf{77}     
        & 57 & 0.47  
        & 63 & 0.44 
        & 66 & 0.41
        & 0.036 \\
    \addlinespace[0.1em]
      FuseMix
        & 64 & 80   & 75     
        & 54 & 0.50 
        & 60 & 0.47 
        & 62 & 0.43
        & {--} \\
    \addlinespace[0.1em]
      \midrule
      Uniform Attention
        & 65 & 80   & \textbf{77}     
        & 59 & \textbf{0.45}  
        & 66 & 0.43 
        & 66 & 0.42
        & 0.02 \\
    \addlinespace[0.1em]
      No Residual Connection
        & 65 & 77   & 75     
        & 55 & 0.48 
        & 64 & 0.47 
        & 65 & 0.45
        & 0.02 \\
    \addlinespace[0.1em]
      No Contrastive Alignment
        & 63 & 69   & 74     
        & 49 & 0.58 
        & 58 & 0.50 
        & 62 & 0.49
        & 0.02 \\
    \addlinespace[0.1em]
      Early Fusion (Mean)
        & 63 & 70   & 73     
        & 58 & 0.49 
        & 61 & 0.49 
        & 63 & 0.48
        & 0.02 \\
    \addlinespace[0.1em]
      \midrule
      \textbf{CLARGA} 
        & \textbf{69} & \textbf{83} & \textbf{77} 
        & \textbf{60} & \textbf{0.45} 
        & \textbf{68} & \textbf{0.41} 
        & \textbf{70} & \textbf{0.40} 
        & {0.02} \\
      \bottomrule
    \end{tabular}%
}
\end{table*}

\begin{table}[t]
  \centering
  \caption{Impact of dropping modalities on AV-MNIST dataset}
  \label{tab:ablation-modalities}
  \renewcommand{\arraystretch}{1.22} %
  \resizebox{\columnwidth}{!}{%
  \begin{tabular}{lccc}
    \toprule
    \textbf{Model}
      & \textbf{Acc.\ (\%)}
      & \multicolumn{2}{c}{\textbf{Kept Modality}} \\
    \cmidrule(lr){3-4}
      & 
      & {Image (\%)} 
      & {Audio (\%)} \\
    \midrule
    FuseMix      & 75 &  --  &  --  \\
    Multimodal Lego       & \textbf{77} &  57 (-20 pp)  &  43 (-34 pp)  \\
    \midrule
    Uniform Attention      & \textbf{77} &  60 (-17 pp)  &  45 (-32 pp)  \\
    No Res. Connection       & 75 &  57 (-18 pp)  &  42 (-33 pp)  \\
    No Contrastive Alignment   & 74 &  54 (-20 pp)  &  38 (-36 pp)  \\
    Early Fusion (Mean)           & 73 &  52 (-21 pp)  &  35 (-38 pp)  \\
    \midrule
    \textbf{CLARGA}   & \textbf{77} & \textbf{62 (-15 pp)} & \textbf{48 (-29 pp)} \\
    \bottomrule
  \end{tabular}
  }
\end{table}

\begin{table}[t]
  \centering
  \caption{Performance against state-of-the-art approaches on the DAIC-WoZ dataset.\protect\footnotemark The first section consists of models evaluated on the AVEC 2016 benchmark \cite{Valstar2016}, while the second consists of K-fold \cite{stone1974cross} or Leave-One-Subject-Out (LOSO) trained models \cite{Kunjan2021}.}
  \label{tab:daic-woz-results}
  \renewcommand{\arraystretch}{1.22} %
  \resizebox{\columnwidth}{!}{%
  \begin{tabular}{lcc}
    \toprule
    \textbf{Model}
      & \textbf{Accuracy (\%)}
      & \textbf{Approach} \\
    \midrule
    (Ma et al., 2016) \cite{Ma2016} & 72.0 & CNN-LSTM \cite{lecun1998gradient, Hochreiter1997} \\
    (Vázquez-Romero et al., 2020) \cite{V_zquez_Romero_2020} & 72.0 & Ensemble CNN \\
    (Muzammel et al., 2021) \cite{muzammel2021end} & 77.2 & LSTM + MLP \cite{Hochreiter1997} \\
    (Patapati, 2024) \cite{patapati2024integratinglargelanguagemodels} & 85.1 & BiLSTM + GPT-4 \cite{cui2019deepbidirectionalunidirectionallstm, openai2024gpt4technicalreport} \\
    \textbf{CLARGA}   & \textbf{91.4} & Adaptive Residual GAT \\
    \midrule
    (Othmani et al., 2022) \cite{OTHMANI2022107132} & 87.4 & VGGish + 1D-CNN \cite{hershey2017cnnarchitectureslargescaleaudio} \\
    (Patapati, 2024) (LOSO) \cite{patapati2024integratinglargelanguagemodels} & 91.0 & BiLSTM + GPT-4 \cite{cui2019deepbidirectionalunidirectionallstm, openai2024gpt4technicalreport} \\
    (Muzammel et al., 2021) (LOSO) \cite{muzammel2021end} & 95.5 & LSTM + MLP \cite{Hochreiter1997} \\
    \textbf{CLARGA (K-Fold)}   & \textbf{95.7} & Adaptive Residual GAT \\
    \bottomrule
  \end{tabular}
  }
\end{table}
\footnotetext{%
We exclude papers which evaluate on sub segments of the DAIC-WoZ, as this has been shown to unfairly inflate accuracy metrics \cite{patapati2024integratinglargelanguagemodels, burdisso23_interspeech}}

We discuss the four complementary studies (as detailed previously in $\S4$): (1) dataset benchmarking (Table \ref{tab:results-wide}), (2) robustness analysis on AV-MNIST (Table \ref{tab:ablation-modalities}), and (3) domain-specific evaluation on DAIC-WoZ (Table \ref{tab:daic-woz-results}).

\subsection{Dataset Benchmarking}
Across seven benchmarks, CLARGA attains (or ties for) the highest accuracy on every task while maintaining very low GFLOPs when compared to the state-of-the-art models. On MM-IMDb, CLARGA improves move-genre prediction accuracy to 69\%, a 3 percent gain over MM-Lego and a 5 percent gain over FuseMix \cite{he2016residual, vouitsis2024dataefficientmultimodalfusionsingle}. For UI topic classification (ENRICO), CLARGA reaches 83\% accuracy, which is a 3\% margin over the next best model and 20\% above the simple early fusion baseline. This demonstrates the effectiveness of CLARGA on smaller datasets.

On the AV-MNIST benchmark, CLARGA performs equal to MM-Lego and the Uniform Attention ablation with an accuracy of 77\%\footnote{Although the performance observed here is very poor compared to what is expected on MNIST \cite{Ciresan2011, LeCun1989}, we would like to emphasize the increased difficulty of AV-MNIST due to the noise, information reduction, and artifacts introduced (as discussed in $\S4.1.1$) \cite{Tseng2023, PerezRua2019}}. Interestingly, we observe that performance of state-of-the-art models, ablations, and CLARGA are all similar when tested on AV-MNIST. We believe this is due to the relatively straightforward and simple nature of the dataset, meaning that much of the performance is based on the capabilities of the encoders themselves. Despite this, CLARGA performs equal to the best and we observe that it possesses the greatest robustness to missing modalities on AV-MNIST, as discussed in $\S5.2$.

CLARGA leads over all other models by 1\%-4\% on STOCKS datasets. We observe that CLARGA scales to higher numbers of modalities\footnote{For the STOCKS datasets we follow (Liang et al., 2021) and the subsequent MultiBench analysis, which treat the return series of each listed company as a separate time-series modality \cite{Liang2021} (see $\S3.1.1$ and $\S3.2$)} very effectively, as its gap in performance compared to other approaches increases as STOCKS datasets incorporate more modalities (100 modalities on STOCKS-TECH versus 18 %
on STOCKS-F\&B).

Additionally, based on results across all the datasets, we can identify that contrastive alignment gives the highest benefit when modalities are semantically distant. On MM-IMDb (image + text) and ENRICO (image + set-like data) the No Contrastive Alignment variant falls behind CLARGA by 6\% and 14\%, respectively. This is consistent with recent papers which demonstrate that InfoNCE lowers statistical distance between heterogeneous embeddings, enabling simpler classifiers to use joint cues linking different modes \cite{Zhou2025MMFRL}.

Put together, the results on these datasets demonstrate CLARGA's ability to adapt to a wide range of modalities (image, audio, text, timeseries) across different domains (pattern recognition, finance, human-computer interaction) very effectively. This is made more impressive when considering the computational complexity of CLARGA, which is relatively low compared to the state-of-the-art approaches and made insignificant when training with larger pipelines.

\subsection{Robustness on AV-MNIST}
Table \ref{tab:ablation-modalities} shows that CLARGA's adaptive design holds up well when one input modality is absent. Removing the image modality, which appears to be the more informative modality \cite{Tseng2023}, reduces CLARGA's accuracy from 77\% to 48\%. While this drop in accuracy is substantial, it is smaller than the declines in the ablations' performance. Early Fusion (Mean) drops in performance by 21\% and even the Uniform Attention variant sees a drop of 17\%. Removing residual connections or the contrastive term reduces performance further. These gains in robustness show that, by using learned edge weights and residual message passing, CLARGA enables the remaining audio information to compensate far more effectively for the missing vision information than static or mean pooling-based approaches. This demonstrates that adaptive weighting and contrastive alignment are the main factors behind this robustness. The pattern is even clearer when the image modality is dropped. CLARGA has a 29\% loss in accuracy, but still outperforms every other model (see Table \ref{tab:ablation-modalities}).

MM-Lego \cite{hemker2025multimodallegomodelmerging} applies "LegoBlocks", wrapping each pretrained encoder in a small adapter that projects its output into a common latent space and then updates it through cross attention with other modalities. When we remove the audio modality, MM-Lego faces a 20\% loss in accuracy. When the image input is removed, its accuracy falls by 34\%. These results place MM-Lego between full CLARGA and the simpler ablations in terms of robustness. In other words, model merging adapters give some level of robustness. However, without CLARGA's adaptive attention (adapted for each node and sample) and contrastive alignment, they cannot match its ability to recover when a modality vanishes.

\subsection{Complex Downstream Task: DAIC-WoZ}
Table \ref{tab:daic-woz-results} shows results on multimodal depression classification. This is a highly difficult and very specific downstream task, where models must pick up on and learn extremely subtle cues correlating to mental health \cite{He2021, Arioz2022}.
The complexity of DAIC-WoZ is demonstrated by the fact that even for models constructed specifically for the dataset, the vast majority do not surpass 80\% accuracy \cite{Valstar2016, muzammel2021end, Ringeval2017, Mamidisetti2022}, and fewer 90\%.

CLARGA attains 91.4\% accuracy on the AVEC-2016 challenge benchmark, surpassing all baselines proposed in the original challenge and recent state-of-the-art models. Under 10 fold cross-validation, CLARGA reaches 95.7\% accuracy, surpassing models evaluated using LOSO.\footnote{The only reason we do not scale to LOSO testing is due to its high compute cost. This biases the comparison in favor of LOSO-based models due to their better access to training data, putting CLARGA at a disadvantage, as demonstrated in previous DAIC-WoZ-based papers \cite{muzammel2021end, patapati2024integratinglargelanguagemodels}

}

We believe that the performance improvement stems from two main factors. Firstly, graph attention balances the uneven predictive strength \cite{muzammel2021end} of the different features and modalities present in the DAIC-WoZ. Secondly, contrastive alignment mitigates the well known scarcity and imbalance of psychiatric data \cite{patapati2024integratinglargelanguagemodels, muzammel2021end, Mamidisetti2022, Casella2025DataGaps}. This level of performance shows CLARGA's ability to adaptively fuse very different modalities and pick up on faint crossmodal patterns. Such patterns are only revealed given the context of multiple separate modalities. %

\section{Conclusion}
CLARGA delivers a %
versatile and efficient framework that consistently excels across diverse multimodal representation learning challenges. By combining adaptive graph attention, residual message passing, and contrastive alignment, it not only achieves state-of-the-art accuracy across several benchmarks but also adapts well to missing inputs. Our theoretical guarantees and low computatonal cost further demonstrate its real world practicality.
These findings show that CLARGA offers a robust and practical approach for general-purpose multimodal representation learning.

{
    \small
    \bibliographystyle{ieeenat_fullname}
    \bibliography{main}
}

\end{document}